\documentclass{article}

\usepackage[final]{nips_2016}

\usepackage[utf8]{inputenc} 
\usepackage[T1]{fontenc}    
\usepackage{hyperref}       
\usepackage{url}            
\usepackage{booktabs}       
\usepackage{amsfonts}       
\usepackage{nicefrac}       
\usepackage{microtype}      
\usepackage{amsmath,amssymb,amsthm,bm,algorithm,algorithmic,graphicx}
\usepackage{epstopdf,color}
\usepackage[usenames,dvipsnames,svgnames,table]{xcolor}
\usepackage{subcaption}

\def\b{\bm b}
\def\W{\mathbf{W}}
\def\S{\mathbf{\Sigma}}
\def\SS{\S_\text{new}}
\def\m{{\bm \mu}}

\def\g{h}
\def\f{g}

\def\d{\delta}
\def\dn{\bm{\delta}}

\def\th{\bm \theta}
\def\tr{^{\top}}
\def\a{\alpha}
\def\reals{\mathbb{R}}
\def\E#1{\mathbb{E}\left[ #1 \right]}
\def\P#1{P\left( #1 \right)}

\def\Id{\mathbf{I}}

\def\ymin{y^{\min}}
\def\ymax{y^{\max}}

\DeclareMathOperator*{\erf}{\textrm{erf}}
\DeclareMathOperator*{\median}{\textrm{median}}
\DeclareMathOperator*{\argmax}{\textrm{argmax}}
\renewcommand\max{\textrm{max}}
\renewcommand\min{\textrm{min}}

\newtheorem{theorem}{Proposition}
\newtheorem{theoremAppendix}{Proposition}

\title{Learning values across many orders of magnitude}

\author{
  Hado van Hasselt
  \And
  Arthur Guez
  \And
  Matteo Hessel\\[2em]
  Google DeepMind
  \And
  Volodymyr Mnih
  \And
  David Silver 
}

\begin{document}

\maketitle

\begin{abstract}
Most learning algorithms are not invariant to the scale of the function that is being approximated.
We propose to adaptively normalize the targets used in learning. 
This is useful in value-based reinforcement learning, where the magnitude of appropriate value approximations can change over time when we update the policy of behavior.
Our main motivation is prior work on learning to play Atari games, where the rewards were all clipped to a predetermined range.
This clipping facilitates learning across many different games with a single learning algorithm, but a clipped reward function can result in qualitatively different behavior.
Using the adaptive normalization we can remove this domain-specific heuristic without diminishing overall performance.
\end{abstract}

\section{Introduction}
Our main motivation is the work by \citet{Mnih:2015}, in which Q-learning \citep{Watkins:1989} is combined with a deep convolutional neural network \citep[cf.][]{LeCun:2015}. The resulting deep Q network (DQN) algorithm learned to play a varied set of Atari 2600 games from the Arcade Learning Environment (ALE) \citep{Bellemare:2013}, which was proposed as an evaluation framework to test general learning algorithms on solving many different interesting tasks.  DQN was proposed as a singular solution, using a single set of hyperparameters, but the magnitudes and frequencies of rewards vary wildly between different games. To overcome this hurdle, the rewards and temporal-difference errors were clipped to $[-1,1]$.  For instance, in Pong the rewards are bounded by $-1$ and $+1$ while in Ms.~Pac-Man eating a single ghost can yield a reward of up to $+1600$, but DQN clips the latter to $+1$ as well. This is not a satisfying solution for two reasons.
First, such clipping introduces domain knowledge. Most games have sparse non-zero rewards outside of $[-1,1]$.  Clipping then results in optimizing the frequency of rewards, rather than their sum.  This is a good heuristic in Atari, but it does not generalize to other domains.
More importantly, the clipping changes the objective, sometimes resulting in qualitatively different policies of behavior.

We propose a method to adaptively normalize the targets used in the learning updates.  If these targets are guaranteed to be normalized it is much easier to find suitable hyperparameters.
The proposed technique is not specific to DQN and is more generally applicable in supervised learning and reinforcement learning. There are several reasons such normalization can be desirable.  First, sometimes we desire a single system that is able to solve multiple different problems with varying natural magnitudes, as in the Atari domain.  Second, for multi-variate functions the normalization can be used to disentangle the natural magnitude of each component from its relative importance in the loss function.  This is particularly useful when the components have different units, such as when we predict signals from sensors with different modalities.  Finally, adaptive normalization can help deal with non-stationary. For instance, in reinforcement learning the policy of behavior can change repeatedly during learning, thereby changing the distribution and magnitude of the values.

\subsection{Related work}
Many machine-learning algorithms rely on a-priori access to data to properly tune relevant hyper-parameters \citep{Bergstra:2011,Bergstra:2012,Snoek:2012}.
However, it is much harder to learn efficiently from a stream of data when we do not know the magnitude of the function we seek to approximate beforehand, or if these magnitudes can change over time, as is for instance typically the case in reinforcement learning when the policy of behavior improves over time.

Input normalization has long been recognized as important to efficiently learn non-linear approximations such as neural networks \citep{LeCun:1998}, leading to research on how to achieve scale-invariance on the inputs \citep[e.g.,][]{Ross:2013,Ioffe:2015,Desjardins:2015}. Output or target normalization has not received as much attention, probably because in supervised learning data is commonly available before learning commences, making it straightforward to determine appropriate normalizations or to tune hyper-parameters. However, this assumes the data is available a priori, which is not true in online (potentially non-stationary) settings.

Natural gradients \citep{Amari:1998} are invariant to reparameterizations of the function approximation, thereby avoiding many scaling issues, but these are computationally expensive for functions with many parameters such as deep neural networks.  This is why approximations are regularly proposed, typically trading off accuracy to computation \citep{Martens:2015}, and sometimes focusing on a certain aspect such as input normalization \citep{Desjardins:2015,Ioffe:2015}.  Most such algorithms are not fully invariant to rescaling the targets.

In the Atari domain several algorithmic variants and improvements for DQN have been proposed \citep{vanHasselt:2016, Bellemare:2016, Schaul:2016, Wang:2016}, as well as alternative solutions \citep{Liang:2016, Mnih:2016}.  However, none of these address the clipping of the rewards or explicitly discuss the impacts of clipping on performance or behavior.

\subsection{Preliminaries}

Concretely, we consider learning from a stream of data $\{ (X_t, Y_t) \}_{t=1}^{\infty}$ where the inputs $X_t \in \reals^n$ and targets $Y_t \in \reals^k$ are real-valued tensors.
The aim is to update parameters $\th$ of a function $f_{\th} : \reals^n \to \reals^k$ such that the output $f_{\th}(X_t)$ is (in expectation) close to the target $Y_t$ according to some loss $l_t(f_{\th})$, for instance defined as a squared difference: $l_t(f_{\th}) = \frac{1}{2}(f_{\th}(X_t) - Y_t)\tr (f_{\th}(X_t) - Y_t)$.
A canonical update is stochastic gradient descent (SGD).
For a sample $(X_t,Y_t)$, the update is then
$
\th_{t+1}
=
\th_t - \a \nabla_{\th} l_t(f_{\th})
$, where $\a \in [0,1]$ is a step size.
The magnitude of this update depends on both the step size and the loss, and it is hard to pick suitable step sizes when nothing is known about the magnitude of the loss.

An important special case is when $f_{\th}$ is a neural network \citep{McCullochPitts:1943,Rosenblatt:62}, which are often trained with a form of SGD \citep{Rumelhart:86}, with hyperparameters that interact with the scale of the loss. Especially for deep neural networks \citep{LeCun:2015,Schmidhuber:2015} large updates may harm learning, because these networks are highly non-linear and such updates may `bump' the parameters to regions with high error.

\section{Adaptive normalization with Pop-Art}
We propose to normalize the targets $Y_t$, where the normalization is learned separately from the approximating function.  We consider an affine transformation of the targets
\begin{equation}\label{norm_Y}
\tilde{Y}_t = \S_t^{-1} (Y_t - \m_t) \,,
\end{equation}
where $\S_t$ and $\m_t$ are \emph{scale} and \emph{shift} parameters that are learned from data. The scale matrix $\S_t$ can be dense, diagonal, or defined by a scalar $\sigma_t$ as $\S_t = \sigma_t \Id$. Similarly, the shift vector $\m_t$ can contain separate components, or be defined by a scalar $\mu_t$ as $\m_t = \mu_t {\bm 1}$.  
We can then define a loss on a normalized function $\f(X_t)$ and the normalized target $\tilde{Y}_t$.  The unnormalized approximation for any input $x$ is then given by
$
f(x) = \S \f(x) + \m
$, where
$\f$ is the \emph{normalized function} and $f$ is the \emph{unnormalized function}.

At first glance it may seem we have made little progress.  If we learn $\S$ and $\m$ using the same algorithm as used for the parameters of the function $\f$, then the problem has not become fundamentally different or easier; we would have merely changed the structure of the parameterized function slightly.  Conversely, if we consider tuning the scale and shift as hyperparameters then tuning them is not fundamentally easier than tuning other hyperparameters, such as the step size, directly.

Fortunately, there is an alternative.  We propose to update $\S$ and $\m$ according to a separate objective with the aim of normalizing the updates for $\f$. Thereby, we decompose the problem of learning an appropriate normalization from learning the specific shape of the function.
The two properties that we want to simultaneously achieve are
\begin{enumerate}
\item[\textbf{(ART)}] to update scale $\S$ and shift $\m$ such that $\S^{-1} (Y - \m)$ is appropriately normalized, and
\item[\textbf{(POP)}] to preserve the outputs of the unnormalized function when we change the scale and shift.
\end{enumerate}
We discuss these properties separately below.
We refer to algorithms that combine output-preserving updates and adaptive rescaling, as \textbf{Pop-Art} algorithms, an acronym for ``Preserving Outputs Precisely, while Adaptively Rescaling Targets''.

\subsection{Preserving outputs precisely}\label{sec_pop}
Unless care is taken, repeated updates to the normalization might make learning harder rather than easier because the normalized targets become non-stationary.  More importantly, whenever we adapt the normalization based on a certain target, this would simultaneously change the output of the unnormalized function of all inputs.  If there is little reason to believe that other unnormalized outputs were incorrect, this is undesirable and may hurt performance in practice, as illustrated in Section \ref{sec_bin}. We now first discuss how to prevent these issues, before we discuss how to update the scale and shift.

The only way to avoid changing all outputs of the unnormalized function whenever we update the scale and shift is by changing the normalized function $\f$ itself simultaneously.  The goal is to preserve the outputs from before the change of normalization, for all inputs.  This prevents the normalization from affecting the approximation, which is appropriate because its objective is solely to make learning easier, and to leave solving the approximation itself to the optimization algorithm.

Without loss of generality the unnormalized function can be written as
\begin{equation}\label{def_f}
f_{\th,\S,\m,\W,\b}(x)
\quad\equiv\quad
\S\f_{\th,\W,\b}(x) + \m
\quad\equiv\quad
\S(\W \g_{\th}(x) + \b) + \m \,,
\end{equation}
where $\g_{\th}$ is a parametrized (non-linear) function, and $\f_{\th,\W,\b} = \W \g_{\th}(x) + \b$ is the normalized function. It is not uncommon for deep neural networks to end in a linear layer, and then $\g_{\th}$ can be the output of the last (hidden) layer of non-linearities.  Alternatively, we can always add a square linear layer to any non-linear function $\g_{\th}$ to ensure this constraint, for instance initialized as $\W_0 = \Id$ and $\b_0 = {\bm 0}$.

The following proposition shows that we can update the parameters $\W$ and $\b$ to fulfill the second desideratum of preserving outputs precisely for any change in normalization.
\def\SS{\S_\mathrm{new}}
\def\mm{\m_\mathrm{new}}
\def\WW{\W_\mathrm{new}}
\def\bb{\b_\mathrm{new}}
\begin{theorem}\label{thm_pop}
Consider a function $f : \reals^n \to \reals^k$ defined as in \eqref{def_f} as
\[
f_{\th,\S,\m,\W,\b}(x)
\quad\equiv\quad
\S \left( \W \g_{\th}(x) + \b \right) + \m \,,
\]
where $\g_{\th} : \reals^n \to \reals^m$ is any non-linear function of $x \in \reals^n$, $\S$ is a $k\times k$ matrix, $\m$ and $\b$ are $k$-element vectors, and $\W$ is a $k \times m$ matrix. Consider any change of the scale and shift parameters from $\S$ to $\SS$ and from $\m$ to $\mm$, where $\SS$ is non-singular. If we then additionally change the parameters $\W$ and $\b$ to $\WW$ and $\bb$, defined by
\begin{align*}
\WW & = \SS^{-1} \S \W
& \text{ and } & & 
\bb & = \SS^{-1}\left( \S \b + \m - \mm \right) \,,
\end{align*}
then the outputs of the unnormalized function $f$ are preserved precisely in the sense that
\[
f_{\th,\S,\m,\W,\b}(x) = f_{\th,\SS,\mm,\WW,\bb}(x)\,,\qquad\forall x \,.
\]
\end{theorem}
This and later propositions are proven in the appendix.
For the special case of scalar scale and shift, with $\S \equiv \sigma\Id$ and $\m \equiv \mu{\bm 1}$, the updates to $\W$ and $\b$ become
$\WW = (\sigma / \sigma_\mathrm{new})\W$ and $\bb = ( \sigma \b + \m - \m_\mathrm{new} )/\sigma_\mathrm{new}$. 
After updating the scale and shift we can update the output of the normalized function $\f_{\th,\W,\b}(X_t)$ toward the normalized output $\tilde{Y}_t$, using any learning algorithm.  Importantly, the normalization can be updated first, thereby avoiding harmful large updates just before they would otherwise occur.  This observation will be made more precise in Proposition \ref{bounded_normal} in Section \ref{sec_Art}.

Algorithm \ref{alg_popart_sgd} is an example implementation of SGD with Pop-Art for a squared loss.  It can be generalized easily to any other loss by changing the definition of $\dn$.
Notice that $\W$ and $\b$ are updated twice: first to adapt to the new scale and shift to preserve the outputs of the function, and then by SGD.  The order of these updates is important because it allows us to use the new normalization immediately in the subsequent SGD update.
\begin{algorithm}[t]
\caption{\label{alg_popart_sgd} SGD on squared loss with Pop-Art}
\begin{algorithmic}
\STATE For a given differentiable function $\g_{\th}$, initialize $\th$. 
\STATE Initialize $\W = I$, $\b = {\bm 0}$, $\S = I$, and $\m = {\bm 0}$.
\WHILE{learning}
\STATE Observe input $X$ and target $Y$
\STATE Use $Y$ to compute new scale $\SS$ and new shift $\mm$
\STATE $\W \gets \SS^{-1} \S \W \,,\quad \b \gets \SS^{-1} ( \S \b + \m - \mm)$\hfill (rescale $\W$ and $\b$)
\STATE $\S~ \gets \SS\,,\phantom{\S\W}\quad \m \gets \mm$ \hfill (update scale and shift)
\STATE ${\bm h} \gets \g_{\th}(X)$ \hfill (store output of $\g_{\th}$)
\STATE ${\bm J} \gets (\nabla_{\th} \g_{\th,1}(X), \ldots, \nabla_{\th} \g_{\th,m}(X))$ \hfill (compute Jacobian of $\g_{\th}$)
\STATE ${\bm \d} \gets \W {\bm h} + \b - \S^{-1}(Y - \m)$ \hfill (compute normalized error)
\STATE $\th \gets \th - \a {\bm J} \W\tr \dn$ \hfill (compute SGD update for $\th$)
\STATE $\W \gets \W - \a \dn {\bm h}\tr$ \hfill (compute SGD update for $\W$)
\STATE $\b \gets \b - \a \dn$ \hfill (compute SGD update for $\b$)
\ENDWHILE
\end{algorithmic}
\end{algorithm}

\subsection{Adaptively rescaling targets}\label{sec_Art}
A natural choice is to normalize the targets to approximately have zero mean and unit variance.  For clarity and conciseness, we consider scalar normalizations.
It is straightforward to extend to diagonal or dense matrices.
If we have data $\{(X_i,Y_i)\}_{i=1}^t$ up to some time $t$, we then may desire
\begin{align}
& \sum_{i=1}^t (Y_i - \mu_t)/\sigma_t = 0
& \text{and} &&
& \frac{1}{t} \sum_{i=1}^t (Y_i - \mu_t)^2 / \sigma_t^2 = 1 \,,\notag\\
\text{such that}\qquad
\mu_t & = \frac{1}{t} \sum_{i=1}^t Y_i
& \text{and} &
& \sigma_t & = \frac{1}{t} \sum_{i=1}^t Y_i^2 - \mu_t^2\,.\label{batch_mean_variance}
\end{align}
This can be generalized to incremental updates
\begin{align}
\mu_t &= ( 1 - \beta_t ) \mu_{t-1} + \beta_t Y_t
\quad\text{and}\quad
\sigma_t^2 & = \nu_t - \mu_t^2 
\,,\quad\text{where} \quad
\nu_t = (1 - \beta_t ) \nu_{t-1} + \beta_t Y_t^2 \,.\label{online_normal}
\end{align}
Here $\nu_t$ estimates the second moment of the targets and $\beta_t \in [0,1]$ is a step size.  If $\nu_t - \mu_t^2$ is positive initially then it will always remain so, although to avoid issues with numerical precision it can be useful to enforce a lower bound explicitly by requiring $\nu_t - \mu_t^2 \ge \epsilon$ with $\epsilon>0$.
For full equivalence to \eqref{batch_mean_variance} we can use $\beta_t = 1/t$.  If $\beta_t = \beta$ is constant we get exponential moving averages, placing more weight on recent data points which is appropriate in non-stationary settings.

A constant $\beta$ has the additional benefit of never becoming negligibly small.  Consider the first time a target is observed that is much larger than all previously observed targets.  If $\beta_t$ is small, our statistics would adapt only slightly, and the resulting update may be large enough to harm the learning.  If $\beta_t$ is not too small, the normalization can adapt to the large target before updating, potentially making learning more robust.  In particular, the following proposition holds.
\begin{theorem}\label{bounded_normal}
When using updates \eqref{online_normal} to adapt the normalization parameters $\sigma$ and $\mu$, the normalized targets are bounded for all $t$ by
\[
- \sqrt{(1 - \beta_t)/\beta_t} \le (Y_t - \mu_t)/\sigma_t \le \sqrt{(1 - \beta_t)/\beta_t} \,.
\]
\end{theorem}
For instance, if $\beta_t = \beta =10^{-4}$ for all $t$, then the normalized target is guaranteed to be in $(-100, 100)$.  Note that Proposition \ref{bounded_normal} does not rely on any assumptions about the distribution of the targets.  This is an important result, because it implies we can bound the potential normalized errors before learning, without any prior knowledge about the actual targets we may observe.

It is an open question whether it is uniformly best to normalize by mean and variance.  In the appendix we discuss other normalization updates, based on percentiles and mini-batches, and derive correspondences between all of these.

\subsection{An equivalence for stochastic gradient descent}
We now step back and analyze the effect of the magnitude of the errors on the gradients when using regular SGD.  This analysis suggests a different normalization algorithm, which has an interesting correspondence to Pop-Art SGD.

We consider SGD updates for an unnormalized multi-layer function of form
$
f_{\th, \W, \b}(X) = \W \g_{\th}(X) + \b
$.
The update for the weight matrix $\W$ is
\[
\W_t = \W_{t-1} + \a_t {\bm \d}_t \g_{\th_t}(X_t)\tr \,,
\]
where ${\bm \d}_t = f_{\th, \W, \b}(X) - Y_t$ is gradient of the squared loss, which we here call the unnormalized error.
The magnitude of this update depends linearly on the magnitude of the error, which is appropriate when the inputs are normalized, because then the ideal scale of the weights depends linearly on the magnitude of the targets.\footnote{In general care should be taken that the inputs are well-behaved; this is exactly the point of recent work on input normalization \citep{Ioffe:2015,Desjardins:2015}.}

Now consider the SGD update to the parameters of $\g_{\th}$,
$
\th_t = \th_{t-1} - \a {\bm J}_t \W_{t-1}\tr {\bm \d}_t
$
where ${\bm J}_t = ( \nabla g_{\th,1}(X), \ldots, \nabla g_{\th,m}(X) )\tr$ is the Jacobian for $\g_{\th}$.
The magnitudes of both the weights $\W$ and the errors ${\bm \d}$ depend linearly on the magnitude of the targets.  This means that the magnitude of the update for $\th$ depends quadratically on the magnitude of the targets.  There is no compelling reason for these updates to depend at all on these magnitudes because the weights in the top layer already ensure appropriate scaling.  In other words, for each doubling of the magnitudes of the targets, the updates to the lower layers quadruple for no clear reason.

This analysis suggests an algorithmic solution, which seems to be novel in and of itself, in which we track the magnitudes of the targets in a separate parameter $\sigma_t$, and then multiply the updates for all lower layers with a factor $\sigma_t^{-2}$.  A more general version of this for matrix scalings is given in Algorithm 2.
We prove an interesting, and perhaps surprising, connection to the Pop-Art algorithm.

\begin{algorithm}[t]
\caption{\label{alg_norm_sgd} Normalized SGD}
\begin{algorithmic}
\STATE For a given differentiable function $\g_{\th}$, initialize $\th$.
\WHILE{learning}
\STATE Observe input $X$ and target $Y$
\STATE Use $Y$ to compute new scale $\S$
\STATE ${\bm h} \gets \g_{\th}(X)$ \hfill (store output of $\g_{\th}$)
\STATE ${\bm J} \gets (\nabla \g_{\th,1}(X), \ldots, \nabla \g_{\th,m}(X))\tr$ \hfill (compute Jacobian of $\g_{\th}$)
\STATE ${\bm \d} \gets \W {\bm h} + \b - Y$ \hfill (compute \emph{un}normalized error)
\STATE $\th \gets \th - \a {\bm J} (\S^{-1} \W)\tr \S^{-1} {\bm \d}$ \hfill (update $\th$ with scaled SGD)
\STATE $\W \gets \W - \a {\bm \d} {\bm g}\tr$ \hfill (update $\W$ with SGD)
\STATE $\b \gets \b - \a {\bm \d}$ \hfill (update $\b$ with SGD)
\ENDWHILE
\end{algorithmic}
\end{algorithm}

\begin{theorem}
Consider two functions defined by
\begin{align*}
f_{\th,\S,\m,\W,\b}(x)
& = \S ( \W \g_{\th}(x) + \b ) + \m &\text{ and}&&
f_{\th,\W,\b}(x)
& = \W \g_{\th}(x) + \b \,,
\end{align*}
where $\g_{\th}$ is the same differentiable function in both cases, and the functions are initialized identically, using $\S_0=\Id$ and $\m = {\bm 0}$, and the same initial $\th_0$, $\W_0$ and $\b_0$.
Consider updating the first function using Algorithm \ref{alg_popart_sgd} (Pop-Art SGD) and the second
using Algorithm \ref{alg_norm_sgd} (Normalized SGD).
Then, for any sequence of non-singular scales $\{\S_t\}_{t=1}^\infty$ and shifts $\{\m_t\}_{t=1}^\infty$, the algorithms are equivalent in the sense that 1) the sequences $\{\th_t\}_{t=0}^\infty$ are identical, 2) the outputs of the functions are identical, for any input.
\end{theorem}
The proposition shows a duality between normalizing the targets, as in Algorithm \ref{alg_popart_sgd}, and changing the updates, as in Algorithm \ref{alg_norm_sgd}.  This allows us to gain more intuition about the algorithm.  In particular, in Algorithm \ref{alg_norm_sgd} the updates in top layer are not normalized, thereby allowing the last linear layer to adapt to the scale of the targets.  This is in contrast to other algorithms that have some flavor of adaptive normalization, such as RMSprop \citep{Tieleman:2012}, AdaGrad \citep{Duchi:2011}, and Adam \citep{Kingma:2015} that each component in the gradient by a square root of an empirical second moment of that component.  That said, these methods are complementary, and it is straightforward to combine Pop-Art with other optimization algorithms than SGD.

\section{Binary regression experiments}\label{sec_bin}
We first analyze the effect of rare events in online learning, when infrequently a much larger target is observed.  Such events can for instance occur when learning from noisy sensors that sometimes captures an actual signal, or when learning from sparse non-zero reinforcements.
We empirically compare three variants of SGD: without normalization, with normalization but without preserving outputs precisely (i.e., with `Art', but without `Pop'), and with Pop-Art.

The inputs are binary representations of integers drawn uniformly randomly between $0$ and $n=2^{10}-1$.  The desired outputs are the corresponding integer values.  Every 1000 samples, we present the binary representation of $2^{16}-1$ as input (i.e., all 16 inputs are 1) and as target $2^{16}-1 = 65,535$.
The approximating function is a fully connected neural network with 16 inputs, 3 hidden layers with 10 nodes per layer, and tanh internal activation functions.  This simple setup allows extensive sweeps over hyper-parameters, to avoid bias towards any algorithm by the way we tune these.  The step sizes $\a$ for SGD and $\beta$ for the normalization are tuned by a grid search over $\{10^{-5}, 10^{-4.5}, \ldots, 10^{-1}, 10^{-0.5}, 1\}$.

\begin{figure}
\begin{subfigure}{.48\textwidth}
\begin{center}
\includegraphics[width=.99\linewidth]{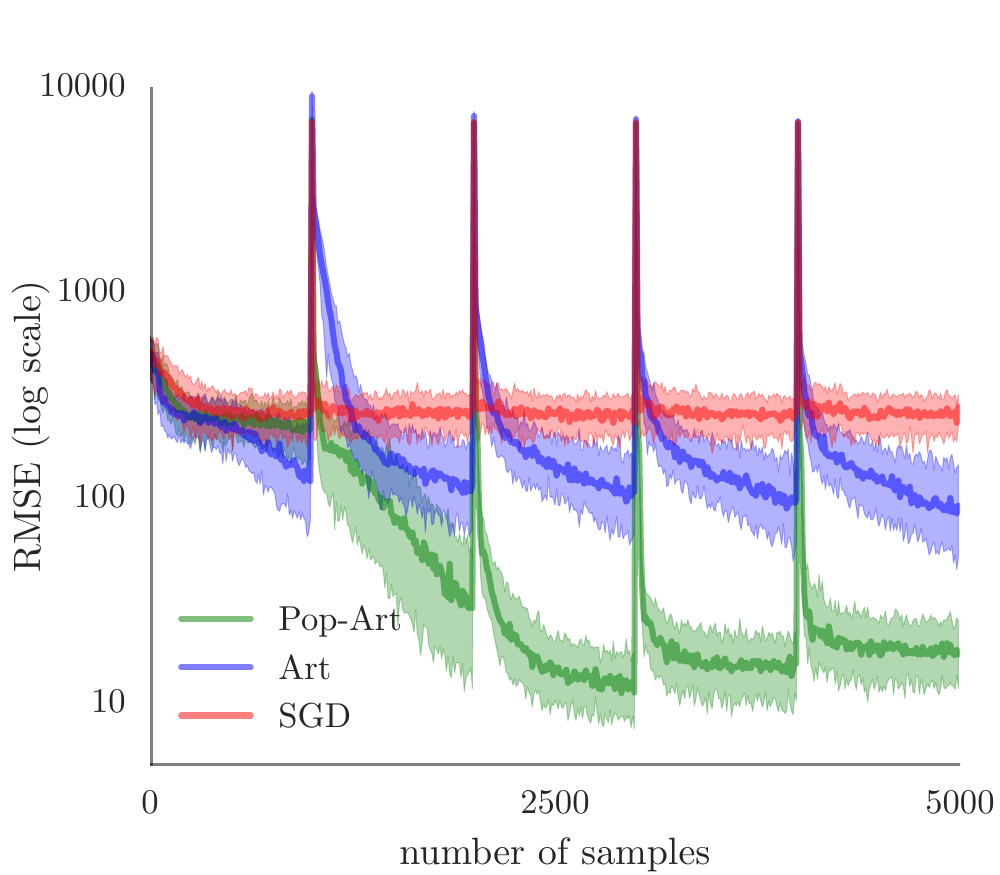}
\caption*{Fig. 1a. \label{binary_popart} Median RMSE on binary regression for SGD without normalization (\textcolor{red}{\textbf{red}}), with normalization but without preserving outputs (\textcolor{blue}{\textbf{blue}}, labeled `Art'), and with Pop-Art (\textcolor{ForestGreen}{\textbf{green}}). Shaded 10--90 percentiles.}
\end{center}
\end{subfigure}%
\begin{subfigure}{.02\textwidth}
\ 
\end{subfigure}
\begin{subfigure}{.48\textwidth}
\begin{center}
\includegraphics[width=.99\linewidth]{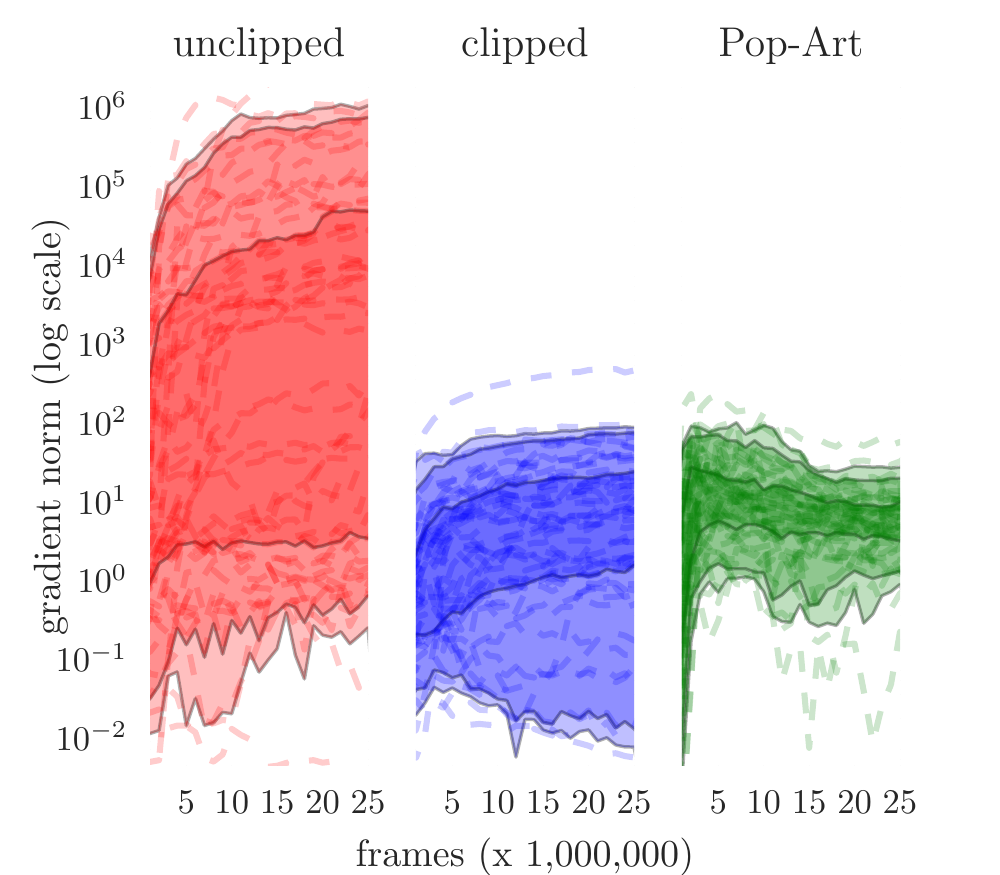}
\caption*{Fig. 1b. \label{fig_norms} $\ell^2$ gradient norms for DQN during learning on 57 Atari games with actual unclipped rewards (left, \textcolor{red}{\textbf{red}}), clipped rewards (middle, \textcolor{blue}{\textbf{blue}}), and using Pop-Art (right, \textcolor{ForestGreen}{\textbf{green}}) instead of clipping.  Shaded areas correspond to 95\%, 90\% and 50\% of games.}
\end{center}
\end{subfigure}
\end{figure}
Figure 1a shows the root mean squared error (RMSE, log scale) for each of 5000 samples, before updating the function (so this is a test error, not a train error).  The solid line is the median of 50 repetitions, and shaded region covers the 10th to 90th percentiles.  The plotted results correspond to the best hyper-parameters according to the overall RMSE (i.e., area under the curve).  The lines are slightly smoothed by averaging over each 10 consecutive samples.

SGD favors a relatively small step size ($\a = 10^{-3.5}$) to avoid harmful large updates, but this slows learning on the smaller updates; the error curve is almost flat in between spikes.  SGD with adaptive normalization (labeled `Art') can use a larger step size ($\a = 10^{-2.5}$) and therefore learns faster, but has high error after the spikes because the changing normalization also changes the outputs of the smaller inputs, increasing the errors on these.  In comparison, Pop-Art performs much better.  It prefers the same step size as Art ($\a = 10^{-2.5}$), but Pop-Art can exploit a much faster rate for the statistics (best performance with $\beta = 10^{-0.5}$ for Pop-Art and $\beta = 10^{-4}$ for Art).  The faster tracking of statistics protects Pop-Art from the large spikes, while the output preservation avoids invalidating the outputs for smaller targets.
We ran experiments with RMSprop but left these out of the figure as the results were very similar to SGD.

\section{Atari 2600 experiments}\label{sec_Atari}
An important motivation for this work is reinforcement learning with non-linear function approximators such as neural networks (sometimes called \emph{deep reinforcement learning}). The goal is to predict and optimize action values defined as the expected sum of future rewards.  These rewards can differ arbitrarily from one domain to the next, and non-zero rewards can be sparse.  As a result, the action values can span a varied and wide range which is often unknown before learning commences.

\citet{Mnih:2015} combined Q-learning with a deep neural network in an algorithm called DQN, which impressively learned to play many games using a single set of hyper-parameters.
However, as discussed above, to handle the different reward magnitudes with a single system all rewards were clipped to the interval $[-1,1]$.  This is harmless in some games, such as Pong where no reward is ever higher than 1 or lower than $-1$, but it is not satisfactory as this heuristic introduces specific domain knowledge that optimizing reward frequencies is approximately is useful as optimizing the total score.  However, the clipping makes the DQN algorithm blind to differences between certain actions, such as the difference in reward between eating a ghost (reward $>= 100$) and eating a pellet (reward $= 25$) in Ms. Pac-Man.
We hypothesize that 1) overall performance decreases when we turn off clipping, because it is not possible to tune a step size that works on many games, 2) that we can regain much of the lost performance by with Pop-Art.  The goal is not to improve state-of-the-art performance, but to remove the domain-dependent heuristic that is induced by the clipping of the rewards, thereby uncovering the true rewards.

We ran the Double DQN algorithm \citep{vanHasselt:2016} in three versions: without changes, without clipping both rewards and temporal difference errors, and without clipping but additionally using Pop-Art.  The targets are the cumulation of a reward and the discounted value at the next state:
\begin{equation}\label{ddqn}
Y_t = R_{t+1} + \gamma Q(S_t,\argmax_a Q(S_t,a;\th) ; \th^-) \,,
\end{equation}
where $Q(s,a;\th)$ is the estimated action value of action $a$ in state $s$ according to current parameters $\th$, and where $\th^-$ is a more stable periodic copy of these parameters \citep[cf.][for more details]{Mnih:2015,vanHasselt:2016}. This is a form of Double Q-learning \citep{vanHasselt:2010, vanHasselt:2011}.
We roughly tuned the main step size and the step size for the normalization to $10^{-4}$.  It is not straightforward to tune the unclipped version, for reasons that will become clear soon. 

Figure 1b shows $\ell^2$ norm of the gradient of Double DQN during learning as a function of number of training steps.  The left plot corresponds to no reward clipping, middle to clipping (as per original DQN and Double DQN), and right to using Pop-Art instead of clipping.  Each faint dashed lines corresponds to the median norms (where the median is taken over time) on one game.  The shaded areas correspond to $50\%$, $90\%$, and $95\%$ of games.

Without clipping the rewards, Pop-Art produces a much narrower band within which the gradients fall.  Across games, $95\%$ of median norms range over less than two orders of magnitude (roughly between 1 and 20), compared to almost four orders of magnitude for clipped Double DQN, and more than six orders of magnitude for unclipped Double DQN without Pop-Art.  The wide range for the latter shows why it is impossible to find a suitable step size with neither clipping nor Pop-Art: the updates are either far too small on some games or far too large on others.

\begin{figure}[t]
\begin{center}
\includegraphics[width=\textwidth]{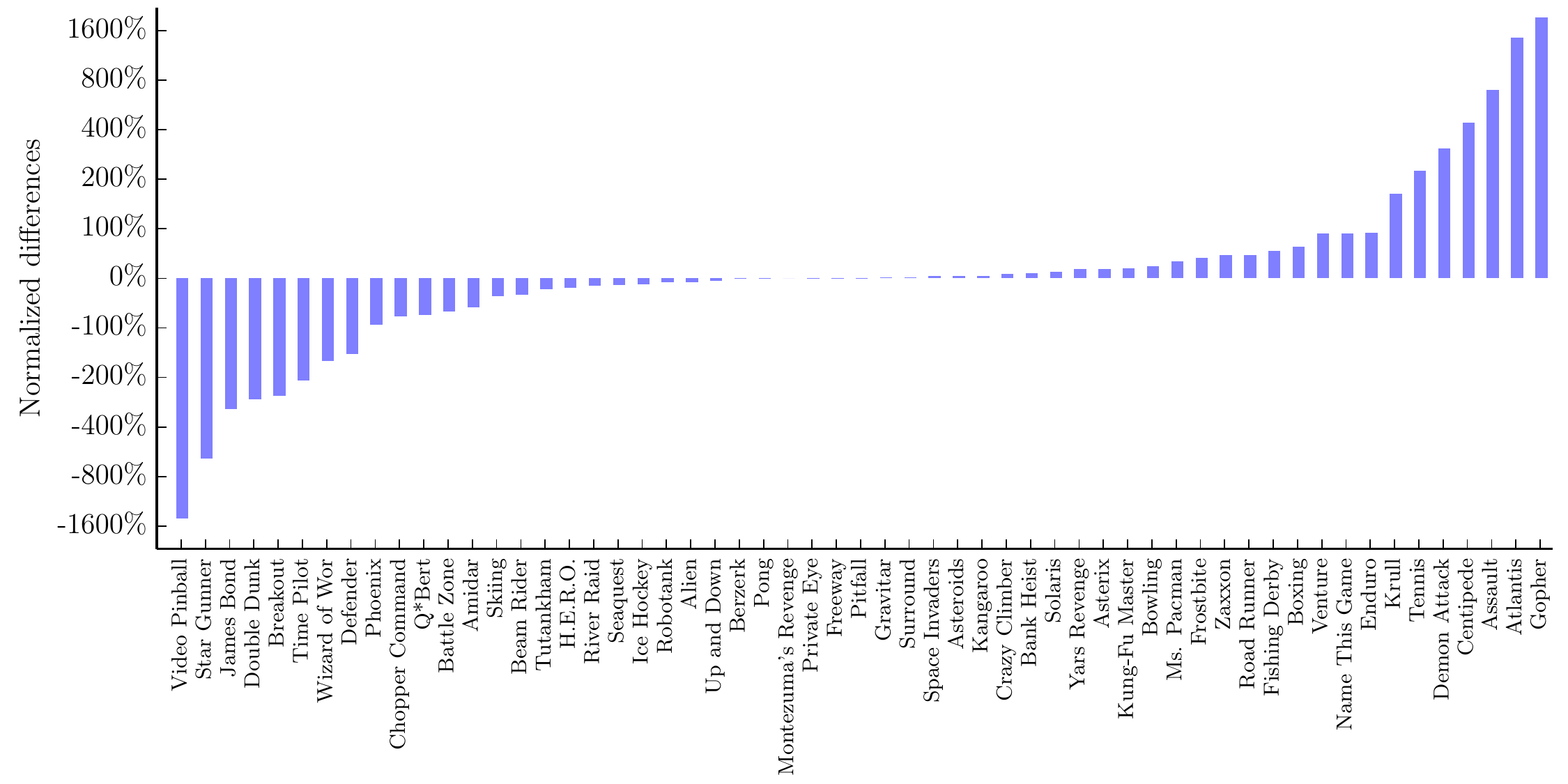}
\caption{\label{fig_Atari} Differences between normalized scores for Double DQN with and without Pop-Art on 57 Atari games.}
\end{center}
\end{figure}

After 200M frames, we evaluated the actual scores of the best performing agent in each game on 100 episodes of up to 30 minutes of play, and then normalized by human and random scores as described by \citet{Mnih:2015}.  Figure~\ref{fig_Atari} shows the differences in normalized scores between (clipped) Double DQN and Double DQN with Pop-Art.

The main eye-catching result is that the distribution in performance drastically changed.  On some games (e.g., Gopher, Centipede) we observe dramatic improvements, while on other games (e.g., Video Pinball, Star Gunner) we see  a substantial decrease.
For instance, in Ms. Pac-Man the clipped Double DQN agent does not care more about ghosts than pellets, but Double DQN with Pop-Art learns to actively hunt ghosts, resulting in higher scores.  Especially remarkable is the improved performance on games like Centipede and Gopher, but also notable is a game like Frostbite which went from below 50\% to a near-human performance level. Raw scores can be found in the appendix.

Some games fare worse with unclipped rewards because it changes the nature of the problem.  For instance, in Time Pilot the Pop-Art agent learns to quickly shoot a mothership to advance to a next level of the game, obtaining many points in the process. The clipped agent instead shoots at anything that moves, ignoring the mothership.\footnote{A video is included in the supplementary material.} However, in the long run in this game more points are scored with the safer and more homogeneous strategy of the clipped agent.  One reason for the disconnect between the seemingly qualitatively good behavior combined with lower scores is that the agents are fairly myopic: both use a discount factor of $\gamma = 0.99$, and therefore only optimize rewards that happen within a dozen or so seconds into the future.
 
On the whole, the results show that with Pop-Art we can successfully remove the clipping heuristic that has been present in all prior DQN variants, while retaining overall performance levels. Double DQN with Pop-Art performs slightly better than Double DQN with clipped rewards: on 32 out of 57 games performance is at least as good as clipped Double DQN and the median (+0.4\%) and mean (+34\%) differences are positive.

\section{Discussion}
We have demonstrated that Pop-Art can be used to adapt to different and non-stationary target magnitudes.  This problem was perhaps not previously commonly appreciated, potentially because in deep learning it is common to tune or normalize a priori, using an existing data set. This is not as straightforward in reinforcement learning when the policy and the corresponding values may repeatedly change over time.  This makes Pop-Art a promising tool for deep reinforcement learning, although it is not specific to this setting.

We saw that Pop-Art can successfully replace the clipping of rewards as done in DQN to handle the various magnitudes of the targets used in the Q-learning update.  Now that the true problem is exposed to the learning algorithm we can hope to make further progress, for instance by improving the exploration \citep{Osband:2016}, which can now be informed about the true unclipped rewards.

\small
\bibliography{../../all}
\bibliographystyle{abbrvnat}
\normalsize

\section*{Appendix}
In this appendix, we introduce and analyze several extensions and variations, including normalizing based on percentiles or minibatches.  Additionally, we prove all propositions in the main text and the appendix.

\subsection*{Experiment setup}
For the experiments described in Section 4 in the main paper, we closely followed the setup described in \citet{Mnih:2015} and \citet{vanHasselt:2016}.  In particular, the Double DQN algorithm is identical to that described by \citeauthor{vanHasselt:2016}\  The shown results were obtained by running the trained agent for 30 minutes of simulated play (or 108,000 frames).  This was repeated 100 times, where diversity over different runs was ensured by a small probability of exploration on each step ($\epsilon$-greedy exploration with $\epsilon=0.01$), as well as by performing up to 30 `no-op' actions, as also used and described by \citeauthor{Mnih:2015}\  In summary, the evaluation setup was the same as used by \citeauthor{Mnih:2015}, except that we allowed more evaluation time per game (30 minutes instead of 5 minutes), as also used by \citet{Wang:2016}.

The results in Figure 2 were obtained by normalizing the raw scores by first subtracting the score by a random agent, and then dividing by the absolute difference between human and random agents, such that
\[
\text{score}^\text{normalized} \equiv \frac{\text{score}^\text{agent} - \text{score}^\text{random}}{|\text{score}^\text{human} - \text{score}^\text{random}|} \,.
\]
The raw scores are given below, in Table \ref{raw_scores}.

\subsection*{Generalizing normalization by variance}
We can change the variance of the normalized targets to influence the magnitudes of the updates.  For a desired standard deviation of $s > 0$, we can use
\[
\sigma_t = \frac{ \sqrt{ \nu_t - \mu_t^2 } }{ s }\,,
\]
with the updates for $\nu_t$ and $\mu_t$ as normal.  It is straightforward to show that then a generalization of Proposition 2 holds with a bound of
\[
- s\sqrt{\frac{1 - \beta_t}{\beta_t}} \le \frac{Y_t - \mu_t}{\sigma_t} \le s\sqrt{\frac{1 - \beta_t}{\beta_t}} \,.
\]
This additional parameter is for instance useful when we desire fast tracking in non-stationary problems.  We then want a large step size $\alpha$, but without risking overly large updates.

The new parameter $s$ may seem superfluous because increasing the normalization step size $\beta$ also reduces the hard bounds on the normalized targets.  However, $\beta$ additionally influences the distribution of the normalized targets.  The histograms in the left-most plot in Figure \ref{fig_hist} show what happens when we try to limit the magnitudes using only $\beta$. The red histogram shows normalized targets where the unnormalized targets come from a normal distribution, shown in blue.  The normalized targets are contained in $[-1,1]$, but the distribution is very non-normal even though the actual targets are normal.  Conversely, the red histogram in the middle plot shows that the distribution remains approximately normal if we instead use $s$ to reduce the magnitudes.  The right plot shows the effect on the variance of normalized targets for either approach.  When we change $\beta$ while keeping $s=1$ fixed, the variance of the normalized targets can drop far below the desired variance of one (magenta curve). When we use change $s$ while keeping $\beta = 0.01$ fixed, the variance remains predictably at approximately $s$ (black line).  The difference in behavior of the resulting normalization demonstrates that $s$ gives us a potentially useful additional degree of freedom.
\definecolor{darkgreen}{rgb}{0,0.7,0}
\definecolor{purple}{rgb}{0.7,0,0.7}
\begin{figure*}
\centering
\includegraphics[width=\textwidth]{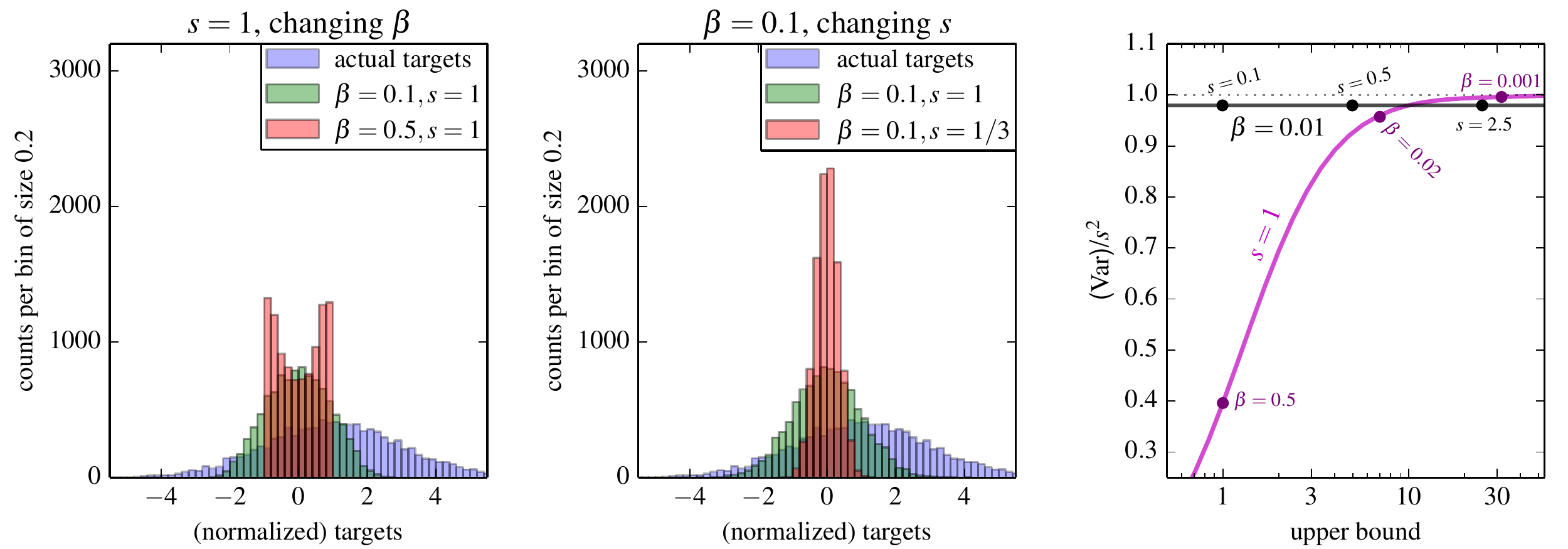}
\caption{\label{fig_hist} The \textbf{left} plot shows histograms for 10,000 normally distributed targets with mean $1$ and standard deviation $2$ (\textbf{\textcolor{blue}{blue}}) and for normalized targets for $\beta=0.1$ (\textbf{\textcolor{darkgreen}{green}}) and $\beta=0.5$ (\textbf{\textcolor{red}{red}}).
The \textbf{middle} plot shows the same histograms, except that the histogram for $\beta=0.5$ and $s=1$ is replaced by a histogram for $\beta=0.1$ and $s=1/3$ (\textbf{\textcolor{red}{red}}).
The \textbf{right} plot shows the variance of the normalized targets as a function of the upper bound $s \sqrt{(1-\beta)/\beta}$ when we either change $\beta$ while keeping $s=1$ fixed (\textbf{\textcolor{purple}{magenta}} curve) or we change $s$ while keeping $\beta=0.01$ fixed (\textbf{black} straight line).
}
\end{figure*}

Sometimes, we can simply roll the additional scaling $s$ into the step size, such that without loss of generality we can use $s=1$ and decrease the step size to avoid overly large updates.  However, sometimes it is easier to separate the magnitude of the targets, as influenced by $s$, from the magnitude of the updates, for instance when using an adaptive step-size algorithm.  In addition, the introduction of an explicit scaling $s$ allows us to make some interesting connections to normalization by percentiles, in the next section.

\subsection*{Adaptive normalization by percentiles}
Instead of normalizing by mean and variance, we can normalize such that a given ratio $p$ of normalized targets is inside the predetermined interval. The per-output objective is then
\[
\text{P}\left( \frac{Y - \mu}{\sigma} \in [-1,1] \right) = p \,.
\]
For normally distributed targets, there is a direct correspondence to normalizing by means and variance.
\begin{theorem}\label{p_to_v}
If scalar targets $\{Y_t\}_{t=1}^\infty$ are distributed according to a normal distribution with arbitrary finite mean and variance, then the objective $\text{P}( ( Y - \mu )/\sigma \in [-1,1] ) = p$ is equivalent to the joint objective $\E{ Y - \mu } = 0$ and $\E{ \sigma^{-2} ( Y - \mu )^2 } = s^2$ with
\[
p
~~=~~ \erf\left( \frac{1}{\sqrt{2}s} \right) \,.
\]
\end{theorem}
For example, percentiles of $p=0.99$ and $p=0.95$ correspond to $s \approx 0.4$ and $s \approx 0.5$, respecticely.  Conversely, $s=1$ corresponds to $p\approx 0.68$.
The fact only applies when the targets are normal. For other distributions the two forms of normalization differ even in terms of their objectives.

We now discuss a concrete algorithm to obtain normalization by percentiles.
Let $Y^{(n)}_t$ denote order statistics of the targets up to time $t$,\footnote{For non-integer $x$ we can define $Y^{(x)}$ by either rounding $x$ to an integer or, perhaps more appropriately, by linear interpolation between the values for the nearest integers.} such that $Y^{(1)}_t = \min_i \{Y_i\}_{i=1}^t$, $Y^{(t)}_t = \max_i \{Y_i\}_{i=1}^t$, and $Y^{((t+1)/2)}_t = \median_i \{Y_i\}_{i=1}^t$.
For notational simplicity, define $n^+ \equiv \frac{t+1}{2}+p\frac{t-1}{2}$ and $n^- \equiv \frac{t+1}{2}-p\frac{t-1}{2}$. Then, for data up to time $t$, the goal is
\begin{align*}
\frac{Y^{(n^+)}_t - \mu_t}{\sigma_t} & = -1 \,,
&\text{and}&&
\frac{Y^{(n^-)}_t - \mu_t}{\sigma_t} = 1 \,.
\end{align*}
Solving for $\sigma_t$ and $\mu_t$ gives
\begin{align*}
\mu_t & = \frac{1}{2} \left( Y^{(n^+)}_t + Y^{(n^-)}_t \right) \,,
&\text{and} &&
\sigma_t & = \frac{1}{2} \left( Y^{(n^+)}_t - Y^{(n^-)}_t \right) \,.
\end{align*}
In the special case where $p=1$ we get $\mu_t = \frac{1}{2}( \max_i Y_i + \min_i Y_i)$ and $\sigma_t = \frac{1}{2} ( \max_i Y_i - \min_i Y_i )$.  We are then guaranteed that all normalized targets fall in $[-1,1]$, but this could result in an overly conservative normalization that is sensitive to outliers and may reduce the overall magnitude of the updates too far.  In other words, learning will then be safe in the sense that no updates will be too big, but it may be slow because many updates may be very small.  In general it is probably typically better to use a ratio $p<1$.

Exact order statistics are hard to compute online, because we would need to store all previous targets.  To obtain more memory-efficient online updates for percentiles we can store two values $\ymin_t$ and $\ymax_t$, which should eventually have the property that a proportion of $(1-p)/2$ values is larger than $\ymax_t$ and a proportion of $(1-p)/2$ values is smaller than $\ymin_t$, such that
\begin{equation}\label{percentile_objective}
\P{ Y > \ymax_t } = \P{ Y < \ymin_t } = (1 - p)/2 \,.
\end{equation}
This can be achieved asymptotically by updating $\ymin_t$ and $\ymax_t$ according to
\begin{align}\label{eq_percent_stats}
\ymax_t & =
\ymax_{t-1} + \beta_t\!\left(\!\mathcal{I}(Y_t > \ymax_{t-1}) - \frac{1-p}{2}\!\right)
&
\text{and}
\\
\ymin_t & =
\ymin_{t-1} - \beta_t\!\left(\!\mathcal{I}(Y_t < \ymin_{t-1}) - \frac{1-p}{2}\right)\!\,,
\end{align}
where the indicator function $\mathcal{I}(\cdot)$ is equal to one when its argument is true and equal to zero otherwise.
\begin{theorem}
If $\sum_{t=1}^\infty \beta_t$ and $\sum_{t=1}^\infty \beta_t^2$, and the distribution of targets is stationary, then the updates in \eqref{eq_percent_stats} converge to values such that \eqref{percentile_objective} holds.
\end{theorem}

If the step size $\beta_t$ is too small it will take long for the updates to converge to appropriate values.  In practice, it might be better to let the magnitude of the steps depend on the actual errors, such that the update takes the form of an asymmetrical least-squares update \citep{Newey:1987,Efron:1991}.

\subsection*{Online learning with minibatches}
Online normalization by mean and variance with minibatches $\{ Y_{t,1}, \ldots, Y_{t,B} \}$ of size $B$ can be achieved by using the updates
\begin{align*}
& \mu_t = (1 - \beta_t) \mu_{t-1} + \beta_t \frac{1}{B} \sum_{b=1}^B Y_{t,b} \;,\;\text{and} & \\
& \sigma_t = \frac{\sqrt{\nu_t - \mu_t^2}}{s} \;,\;\text{where}\\
& \nu_t = (1 - \beta_t) \nu_{t-1} + \beta_t \frac{1}{B} \sum_{b=1}^B Y_{t,b}^2 \,.
\end{align*}
Another interesting possibility is to update $\ymin_t$ and $\ymax_t$ towards the extremes of the minibatch such that
\begin{align}\label{percentile_minibatch_stats}
\ymin_t & = (1 - \beta_t) \ymin_{t-1} + \beta_t \min_b Y_{t,b} \,,
\quad\text{and} \\
\ymax_t & = (1 - \beta_t) \ymax_{t-1} + \beta_t \max_b Y_{t,b} \,,\notag
\end{align}
and then use
\begin{align*}
\mu_t & = \frac{1}{2}( \ymax_t + \ymin_t) \,,
& \quad\text{and}
& & 
\sigma_t & = \frac{1}{2}( \ymax_t - \ymin_t ) \,.
\end{align*}
The statistics of this normalization depend on the size of the minibatches, and there is an interesting correspondence to normalization by percentiles.
\begin{theorem}\label{B_to_p}
Consider minibatches $\{ \{ Y_{t,1}, \ldots, Y_{t,B} \}\}_{t=1}^\infty$ of size $B \ge 2$ whose elements are drawn i.i.d. from a uniform distribution with support on $[a,b]$. If $\sum_t \beta_t = \infty$ and $\sum_t \beta_t^2 < \infty$, then in the limit the updates~\eqref{percentile_minibatch_stats} converge to values such that \eqref{percentile_objective} holds, with $p = (B-1)/(B+1)$.
\end{theorem}
This fact connects the online minibatch updates \eqref{percentile_minibatch_stats} to normalization by percentiles.  For instance, a minibatch size of $B=20$ would correspond roughly to online percentile updates with $p=19/21 \approx 0.9$ and, by Proposition \ref{p_to_v}, to a normalization by mean and variance with a $s \approx 0.6$. These different normalizations are not strictly equivalent, but may behave similarly in practice.

Proposition \ref{B_to_p} quantifies an interesting correspondence between minibatch updates and normalizing by percentiles.  Although the fact as stated holds only for uniform targets, the proportion of normalized targets in the interval $[-1,1]$ more generally becomes larger when we increase the minibatch size, just as when we increase $p$ or decrease $s$, potentially resulting in better robustness to outliers at the possible expense of slower learning.

\subsection*{A note on initialization}
When using constant step sizes it is useful to be aware of the start of learning, to trust the data rather than arbitrary initial values.  This can be done by using a step size as defined in the following fact.
\begin{theorem}
Consider a recency-weighted running average $\bar z_t$ updated from a stream of data $\{Z_t\}_{t=1}^\infty$ using $\bar z_t = 
(1 - \beta_t)\bar z_{t-1} + \beta_t Z_t$, 
with
$\beta_t$ defined by
\begin{equation}\label{better_beta}
\beta_t = \beta(1 - (1 - \beta)^t)^{-1} \,.
\end{equation}
Then 1) the relative weights of the data in $Z_t$ are the same as when using a constant step size $\beta$, and 2) the estimate $\bar z_t$ does not depend on the initial value $\bar z_0$.
\end{theorem}
A similar result was derived to remove the effect of the initialization of certain parameters by \citet{Kingsma:2014} for a stochastic optimization algorithm called Adam. In that work, the initial values are assumed to be zero and a standard exponentially weighted average is explicitly computed and stored, and then divided by a term analogous to $1 - (1 - \beta)^t$.  The step size \eqref{better_beta} corrects for any initialization in place, without storing auxiliary variables, but for the rest the method and its motivation are very similar.

Alternatively, it is possible to initialize the normalization safely, by choosing a scale that is relatively high initially.  This can be beneficial when at first the targets are relatively small and noisy.  If we would then use the step size in \eqref{better_beta}, the updates would treat these initial observations as important, and would try to fit our approximating function to the noise.  A high initialization (e.g., $\nu_0 = 10^4$ or $\nu_0 = 10^6$) would instead reduce the effect of the first targets on the learning updates, and would instead use these only to find an appropriate normalization.  Only after finding this normalization the actual learning would then commence.

\subsection*{Deep Pop-Art}
Sometimes it makes sense to apply the normalization not to the output of the network, but at a lower level.  For instance, the $i^{\text{th}}$ output of a neural network with a soft-max on top can be written
\[
\f_{\th,i}(X) = \frac{\text{e}^{[\W\g_{\th}(X) + \b]_i}}{\sum_j \text{e}^{[\W\g_{\th}(X) + \b]_j}} \,,
\]
where $\W$ is the weight matrix of the last linear layer before the soft-max.  The actual outputs are already normalized by using the soft-max, but the outputs $\W\g_{\th}(X) + \b$ of the layer below the soft-max may still benefit from normalization.  To determine the targets to be normalized, we can either back-propagate the gradient of our loss through the soft-max or invert the function.

More generally, we can consider applying normalization at any level of a hierarchical non-linear function. This seems a promising way to counteract undesirable characteristics of back-propagating gradients, such as vanishing or exploding gradients \citep{Hochreiter:1998}.

In addition, normalizing gradients further down in a network can provide a straightforward way to combine gradients from different sources in more complex network graphs than a standard feedforward multi-layer network.  First, the normalization allows us to normalize the gradient from each source separately before merging gradients, thereby avoiding one source to fully drown out any others and allowing us to weight the gradients by actual relative importance, rather than implicitly relying on the current magnitude of each as a proxy for this.  Second, the normalization can prevent undesirably large gradients when many gradients come together at one point of the graph, by normalizing again after merging gradients.

\section*{Proofs}
\begin{theoremAppendix}
Consider a function $f : \reals^n \to \reals^k$ defined by 
\[
f_{\th,\S,\m,\W,\b}(x)
\quad\equiv\quad
\S \left( \W \g_{\th}(x) + \b \right) + \m \,,
\]
where $\g_{\th} : \reals^n \to \reals^m$ is any non-linear function of $x \in \reals^n$, $\S$ is a $k\times k$ matrix, $\m$ and $\b$ are $k$-element vectors, and $\W$ is a $k \times m$ matrix. Consider any change of the scale and shift parameters from $\S$ to $\S_2$ and from $\m$ to $\m_2$, where $\S_2$ is non-singular. If we then additionally change the parameters $\W$ and $\b$ to $\W_2$ and $\b_2$, defined by
\begin{align*}
\W_2 & = \S^{-1}_2 \S \W
& \text{ and } & & 
\b_2 & = \S^{-1}_2\left( \S \b + \m - \m_2 \right) \,,
\end{align*}
then the outputs of the unnormalized function $f$ are preserved precisely in the sense that
\[
f_{\th,\S,\m,\W,\b}(x) = f_{\th,\S_2,\m_2,\W_2,\b_2}(x)\,,\qquad\forall x \,.
\]
\end{theoremAppendix}
\begin{proof}
The stated result follows from
\begin{align*}
f_{\th,\S_2,\m_2,\W_2,\b_2}(x) 
& = \S_2 \f_{\th,\W_2,\b_2}(x) + \m_2 \\
& = \S_2 \left( \W_2 \g_{\th}(x) + \b_2 \right) + \m_2 \\
& = \S_2 \left( \S_2^{-1} \S \W \g_{\th}(x) + \S^{-1}_2 \left( \S \b + \m - \m_2 \right) \right) + \m_2 \\
& = \left( \S \W \g_{\th}(x) + \S \b + \m - \m_2 \right) + \m_2 \\
& = \S \W \g_{\th}(x) + \S \b + \m \\
& = \S \f_{\th,\W,\b}(x) + \m \\
& = f_{\th,\S,\m,\W,\b}(x) \,.\qedhere
\end{align*}
\end{proof}
\begin{theoremAppendix}
When using updates (4) to adapt the normalization parameters $\sigma$ and $\mu$, the normalized target $\sigma_t^{-1} ( Y_t - \mu_t )$ is bounded for all $t$ by
\[
- \sqrt{\frac{1 - \beta_t}{\beta_t}} \le \frac{Y_t - \mu_t}{\sigma_t} \le \sqrt{\frac{1 - \beta_t}{\beta_t}} \,.
\]
\end{theoremAppendix}
\begin{proof}
\begin{align*}
\left( \frac{ Y_t - \mu_t }{ \sigma_t }\right)^2 
& = \left( \frac{ Y_t - ( 1 - \beta_t ) \mu_{t-1} - \beta_t Y_t }{ \sigma_t } \right)^2\\
& = \frac{(1 - \beta_t)^2 ( Y_t - \mu_{t-1} )^2 }{ \nu_t - \mu_t^2 } \\
& = \frac{(1 - \beta_t)^2 ( Y_t - \mu_{t-1} )^2 }{ (1 - \beta_t ) \nu_{t-1} + \beta_t Y_t^2 - \left( ( 1 - \beta_t ) \mu_{t-1} + \beta_t Y_t \right)^2 } \\
& = \frac{(1 - \beta_t)^2 ( Y_t - \mu_{t-1} )^2 }{ (1 - \beta_t ) \left( \nu_{t-1} + \beta_t Y_t^2 - ( 1 - \beta_t ) \mu_{t-1}^2 - 2 \beta_t \mu_{t-1} Y_t \right) } \\
& = \frac{(1 - \beta_t) ( Y_t - \mu_{t-1} )^2 }{ \nu_{t-1} + \beta_t Y_t^2 - ( 1 - \beta_t ) \mu_{t-1}^2 - 2 \beta_t \mu_{t-1} Y_t  } \\
& \le \frac{(1 - \beta_t) ( Y_t - \mu_{t-1} )^2 }{ \mu_{t-1}^2 + \beta_t Y_t^2 - ( 1 - \beta_t ) \mu_{t-1}^2 - 2 \beta_t \mu_{t-1} Y_t  } \\
& = \frac{(1 - \beta_t) ( Y_t - \mu_{t-1} )^2 }{ \beta_t Y_t^2 + \beta_t \mu_{t-1}^2 - 2 \beta_t \mu_{t-1} Y_t  } \\
& = \frac{(1 - \beta_t) ( Y_t - \mu_{t-1} )^2 }{ \beta_t (Y_t - \mu_{t-1})^2 }\\
& = \frac{(1 - \beta_t) }{ \beta_t } \,,
\end{align*}
The inequality follows from the fact that $\nu_{t-1} \ge \mu_{t-1}^2$.
\end{proof}

\begin{theoremAppendix}
Consider two functions defined by
\begin{align*}
f_{\th,\S,\m,\W,\b}(x)
& = \S ( \W \g_{\th}(x) + \b ) + \m \,,\text{ and}\\
f_{\th,\W,\b}(x)
& = \W \g_{\th}(x) + \b \,,
\end{align*}
where $\g_{\th}$ is the same differentiable function in both cases, and the functions are initialized identically, using $\S_0=\Id$ and $\m = {\bm 0}$, and the same initial $\th_0$, $\W_0$ and $\b_0$.
Consider updating the first function using Algorithm 1 and the second
using Algorithm 2.
Then, for any sequence of non-singular scales $\{\S_t\}_{t=1}^\infty$ and shifts $\{\m_t\}_{t=1}^\infty$, the algorithms are equivalent in the sense that 1) the sequences $\{\th_t\}_{t=0}^\infty$ are identical, 2) the outputs of the functions are identical, for any input.
\end{theoremAppendix}
\begin{proof}
Let $\th_t^1$ and $\th_t^2$ denote the parameters of $\g_{\th}$ for Algorithms 1 and 2, respectively.  Similarly, let $\W^1$ and $\b^1$ be parameters of the first function, while $\W^2$ and $\b^2$ are parameters of the second function.
It is enough to show that single updates of both Algorithms 1 and 2 from the same starting points have equivalent results.  That is, if 
\begin{align*}
\th^2_{t-1} & = \th^1_{t-1} \;,\; \text{and} \\
f_{\th^2_{t-1},\W^2_{t-1},\b^2_{t-1}}(x)
& = f_{\th^1_{t-1},\S_{t-1},\m_{t-1},\W^1_{t-1},\b^1_{t-1}}(x) \,,
\end{align*}
then it must follow that
\begin{align*}
\th^2_{t} & = \th^1_{t} \;,\;\text{and} \\
f_{\th^2_{t},\W^2_{t},\b^2_{t}}(x)
& = f_{\th^1_{t},\S_{t},\m_{t},\W^1_{t},\b^1_{t}}(x) \,,
\end{align*}
where the quantities $\th^2$, $\W^2$, and $\b^2$ are updated with Algorithm 2 and quantities $\th^1$, $\W^1$, and $\b^1$ are updated with Algorithm 1.
We do not require $\W^2_t = \W^1_t$ or $\b^2_t = \b^1_t$, and indeed these quantities will generally differ.

We use the shorthands $f^1_t$ and $f^2_t$ for the first and second function, respectively.
First, we show that $\W^1_t = \S^{-1}_t \W^2_t$, for all $t$.  For $t=0$, this holds trivially because $\W^1_0 = \W^2_0 = \W_0$, and $\S_0 = {\bm I}$.  Now assume that $\W^1_{t-1} = \S^{-1}_{t-1} \W^2_{t-1}$.  Let $\d_t = Y_t - f^1_t(X_t)$ be the unnormalized error at time $t$. 
Then, Algorithm 1 results in
\begin{align*}
\W^1_t
& = \S_t^{-1} \S_{t-1} \W^1_{t-1} + \a \S_t^{-1} \d_t g_{\th_{t-1}}(X_t)\tr \\
& = \S_t^{-1} \left( \S_{t-1} \W^1_{t-1} + \a \d_t g_{\th_{t-1}}(X_t)\tr \right) \\
& = \S_t^{-1} \left( \W^2_{t-1} + \a \d_t g_{\th_{t-1}}(X_t)\tr \right) \\
& = \S_t^{-1} \W^2_t \,.
\end{align*}
Similarly, $\b^1_0 = \S_0^{-1} ( \b^2_0 - \mu_0 )$ and if $\b^1_{t-1} = \S_{t-1}^{-1} ( \b^2_{t-1} - \mu_{t-1} )$ then
\begin{align*}
\b^1_t
& = \S_t^{-1} (\S_{t-1} \b^1_{t-1} + \mu_{t-1} - \mu_t) + \a \S_t^{-1} \d_t \\
& = \S_t^{-1} (\b^2_{t-1} - \mu_t) + \a \S_t^{-1} \d_t \\
& = \S_t^{-1} (\b^2_{t-1} - \mu_t + \a  \d_t )\\
& = \S_t^{-1} (\b^2_t - \mu_t  )\,.
\end{align*}
Now, assume that $\th^1_{t-1} = \th^2_{t-1}$. Then,
\begin{align*}
\th^1_t & = \th^1_{t-1} + \a {\bm J}_t (\W^1_{t-1})\tr \S^{-1}_t \d \\
& = \th^2_{t-1} + \a {\bm J}_t (\S_{t-1}^{-1} \W^2_{t-1})\tr \S^{-1}_t \d \\
& = \th^2_t \,.
\end{align*}
As $\th^1_0 = \th^2_0$ by assumption, $\th_t^1 = \th_t^2$ for all $t$.

Finally, we put everything together and note that $f^1_0 = f^2_0$ and that
\begin{align*}
f^1_t(x)
& = \S_t ( \W^1_t \g_{\th^1_t}(x) + \b^1_t ) + \m_t \\
& = \S_t ( \S_t^{-1} \W^2_t \g_{\th^2_t}(x) + \S_t^{-1}( \b^2_t - \m_t ) ) + \m_t \\
& = \W^2_t \g_{\th^2_t}(x) + \b^2_t \\
& = f^2_t(x) \quad \forall x, t \,. \qedhere
\end{align*}

\end{proof}

\begin{theoremAppendix}
If the targets $\{Y_t\}_{t=1}^\infty$ are distributed according to a normal distribution with arbitrary finite mean and variance, then the objective $\text{P}( \sigma^{-1} ( Y - \mu ) \in [-1,1] ) = p$ is equivalent to the joint objective $\E{ Y - \mu } = 0$ and $\E{ \sigma^{-2} ( Y - \mu )^2 } = s^2$ for
\[
p
~~=~~ \erf\left( \frac{1}{\sqrt{2} s} \right)
\]
\end{theoremAppendix}
\begin{proof}
For any $\mu$ and $\sigma$, the normalized targets are distributed according to a normal distribution because the targets themselves are normally distributed and the normalization is an affine transformation.  For a normal distribution with mean zero and variance $v$, the values $1$ and $-1$ are both exactly $1/\sqrt{v}$ standard deviations from the mean, implying that the ratio of data between these points is $\Phi(1/\sqrt{v}) - \Phi(-1/\sqrt{v})$, where
\[
\Phi(x) = \frac{1}{2} \left( 1 + \erf\left( \frac{x}{\sqrt{2}} \right) \right) \,
\]
is the standard normal cumulative distribution.
The normalization by mean and variance is then equivalent to a normalization by percentiles with a ratio $p$ defined by 
\begin{align*}
& p
~~=~~ \Phi(1/\sqrt{v}) - \Phi(-1/\sqrt{n})\\
& =~~ \frac{1}{2} \left( 1 + \erf \left( \frac{1}{\sqrt{2 v}} \right) \right) - \frac{1}{2} \left( 1 + \erf \left( - \frac{1}{\sqrt{2 v}} \right) \right)\\
& =~~ \erf \left( \frac{1}{\sqrt{2 v}} \right)\,,
\end{align*}
where we used the fact that $\erf$ is odd, such that $\erf(x) = -\erf(-x)$.
\end{proof}
\begin{theoremAppendix}
If $\sum_{t=1}^\infty \beta_t$ and $\sum_{t=1}^\infty \beta_t^2$, and the distribution of targets is stationary, then the updates
\begin{align*}
\ymax_t & =
\ymax_{t-1} + \beta_t \left(\mathcal{I}(Y_t > \ymax_{t-1}) - \frac{1-p}{2}\right)
\quad\text{and}\\
\ymin_t & =
\ymin_{t-1} - \beta_t \left(\mathcal{I}(Y_t < \ymin_{t-1}) - \frac{1-p}{2}\right)\,,
\end{align*}
converge to values such that 
\[
\P{ Y > \ymax_t } = \P{ Y < \ymin_t } = \frac{1 - p}{2} \,.
\]
\end{theoremAppendix}
\begin{proof}
Note that 
\begin{align*}
\E{ \ymax_t = \ymax_{t-1} }
& \iff
\E{ \mathcal{I}(Y > \ymax_t) } = (1 - p)/2\\
& \iff
\P{ Y > \ymax_t } = (1 - p)/2 \,,
\end{align*}
so this is a fixed point of the update.
Note further that the variance of the stochastic update is finite, and that the expected direction of the updates is towards the fixed point, so that this fixed point is an attractor.  The conditions on the step sizes ensure that the fixed point is reachable ($\sum_{t=1}^\infty \beta_t = \infty$) and that we converge upon it in the limit ($\sum_{t=1}^\infty \beta_t^2 < \infty$).  For more detail and weaker conditions, we refer to reader to the extensive literature on stochastic approximation \citep{Robbins:1951,Kushner:2003}.  The proof for the update for $\ymin_t$ is exactly analogous.
\end{proof}

\begin{theoremAppendix}
Consider minibatches $\{ \{ Y_{t,1}, \ldots, Y_{t,B} \}\}_{t=1}^\infty$ of size $B \ge 2$ whose elements are drawn i.i.d. from a uniform distribution with support on $[a,b]$. If $\sum_t \beta_t = \infty$ and $\sum_t \beta_t^2 < \infty$, then in the limit the updates
\begin{align*}
\ymin_t & = (1 - \beta_t) \ymin_{t-1} + \beta_t \min_b Y_{t,b} \,,
\quad\text{and}\\
\ymax_t & = (1 - \beta_t) \ymax_{t-1} + \beta_t \max_b Y_{t,b} \,
\end{align*}
converge to values such that
\[
\P{ Y > \ymax_t } = \P{ Y < \ymin_t } = (1 - p)/2 \,,
\]
with $p = (B-1)/(B+1)$.
\end{theoremAppendix}
\begin{proof}
Because of the conditions on the step size, the quantities $\ymin_t$ and $\ymax_t$ will converge to the expected value for the minimum and maximum of a set of $B$ i.i.d. random variables.
The cumulative distribution function (CDF) for the maximum of $B$ i.i.d. random variables with CDF $F(x)$ is $F(x)^B$, since
\[
P(x < \max_{1\le b \le B} Y_b) = \prod_{b=1}^B P( x < Y_b ) = F(x)^B
\]
The CDF for a uniform random variables with support on $[a,b]$ is
\[
F(x) = 
\left\{\begin{array}{ll}
0 & \text{if $x < a$,}\\
\frac{x - a}{b - a} & \text{if $a \le x \le b$,}\\
1 & \text{if $x > b$.}
\end{array}\right.
\]
Therefore, 
\[
P( x \le \max_{1\le b \le B} Y_b ) =
\left\{\begin{array}{ll}
0 & \text{if $x < a$,}\\
\left(\frac{x - a}{b - a}\right)^B & \text{if $a \le x \le b$,}\\
1 & \text{if $x > b$.}
\end{array}\right.
\]
The associated expected value can then be calculated to be
\[
\E{ \max_{1\le b \le B} Y_b } = a + \frac{B}{B+1} ( b - a ) \,,
\]
so that a fraction of $\frac{1}{B+1}$ of samples will be larger than this value.
Through a similar reasoning, an additional fraction of $\frac{1}{B+1}$ will be smaller than the minimum, and a ratio of $p= \frac{B-1}{B+1}$ will on average fall between these values.
\end{proof}

\begin{theoremAppendix}
Consider a weighted running average $x_t$ updated from a stream of data $\{Z_t\}_{t=1}^\infty$ using
\[
(1 - \beta_t)x_{t-1} + \beta_t Z_t \,,
\]
with
\[
\beta_t \equiv \frac{\beta}{1 - (1 - \beta)^t} \,,
\]
where $\beta$ is a constant.
Then 1) the relative weights of the data in $x_t$ are the same as when only the constant step size $\beta$ is used, and 2) the average does not depend on the initial value $x_0$.
\end{theoremAppendix}
\begin{proof}
The point of the fact is to show that
\begin{equation}\label{init_beta}
\mu_t
= \frac{\mu_t^\beta - (1 - \beta)^t\mu_0}{1 - (1 - \beta)^t} \,,\quad\forall t\,,
\end{equation}
where
\[
\mu_t^\beta = ( 1- \beta)\mu_{t-1}^\beta + \beta Z_t \,,\quad\forall t \,,
\]
and where $\mu_0 = \mu_0^\beta$.  Note that $\mu_t$ as defined by \eqref{init_beta} exactly removes the contribution of the initial value $\mu_0$, which at time $t$ have weight $(1 - \beta)^t$ in the exponential moving average $\mu_t^\beta$, and then renormalizes the remaining value by dividing by $1 - (1 - \beta)^t$, such that the relative weights of the observed samples $\{Z_t\}_{t=1}^\infty$ is conserved.

If \eqref{init_beta} holds for $\mu_{t-1}$, then
\begin{align*}
\mu_t
& = \left( 1 - \frac{\beta}{1 - (1 - \beta)^t} \right)\mu_{t-1} + \frac{\beta}{1 - (1 - \beta)^t} Y_t \\
& = \left( 1 - \frac{\beta}{1 - (1 - \beta)^t} \right)\frac{\mu_{t-1}^\beta - (1 - \beta)^{t-1} \mu_0}{1 - (1 - \beta)^{t-1}}\\
& \qquad + \frac{\beta}{1 - (1 - \beta)^t} Y_t \\
& = \frac{1 - (1 - \beta)^t - \beta}{1 - (1 - \beta)^t} \frac{\mu_{t-1}^\beta - (1 - \beta)^{t-1} \mu_0}{1 - (1 - \beta)^{t-1}} \\
& \qquad + \frac{\beta}{1 - (1 - \beta)^t} Y_t \\
& = \frac{(1 - \beta) ( 1 - (1 - \beta)^{t-1})}{1 - (1 - \beta)^t} \frac{\mu_{t-1}^\beta - (1 - \beta)^{t-1} \mu_0}{1 - (1 - \beta)^{t-1}} \\
& \qquad + \frac{\beta}{1 - (1 - \beta)^t} Y_t \\
& = \frac{(1 - \beta)( \mu_{t-1}^\beta - (1 - \beta)^{t-1}\mu_0)}{1 - (1 - \beta)^t} + \frac{\beta}{1 - (1 - \beta)^t} Y_t \\
& = \frac{(1 - \beta) \mu_{t-1}^\beta + \beta Y_t - (1 - \beta)^t \mu_0}{1 - (1 - \beta)^t} \\
& = \frac{\mu_t^\beta - (1 - \beta)^t \mu_0}{1 - (1 - \beta)^t} \,,
\end{align*}
so that then \eqref{init_beta} holds for $\mu_t$.
Finally, verify that $\mu_1 = Y_1$. Therefore, \eqref{init_beta} holds for all $t$ by induction.
\end{proof}

\begin{table*}
\begin{center}
\small
\begin{tabular}{lrrrr}
Game & Random & Human & Double DQN & Double DQN with Pop-Art \\
\hline
               Alien & $   227.80$ &  $  7127.70$ &  $\mathbf{   3747.70 }$ &  $           3213.50 $ \\
              Amidar & $     5.80$ &  $  1719.50$ &  $\mathbf{   1793.30 }$ &  $            782.50 $ \\
             Assault & $   222.40$ &  $   742.00$ &  $            5393.20 $ &  $\mathbf{   9011.60 }$ \\
             Asterix & $   210.00$ &  $  8503.30$ &  $           17356.50 $ &  $\mathbf{  18919.50 }$ \\
           Asteroids & $   719.10$ &  $ 47388.70$ &  $             734.70 $ &  $\mathbf{   2869.30 }$ \\
            Atlantis & $ 12850.00$ &  $ 29028.10$ &  $          106056.00 $ &  $\mathbf{ 340076.00 }$ \\
          Bank Heist & $    14.20$ &  $   753.10$ &  $            1030.60 $ &  $\mathbf{   1103.30 }$ \\
         Battle Zone & $  2360.00$ &  $ 37187.50$ &  $\mathbf{  31700.00 }$ &  $           8220.00 $ \\
          Beam Rider & $   363.90$ &  $ 16926.50$ &  $\mathbf{  13772.80 }$ &  $           8299.40 $ \\
             Berzerk & $   123.70$ &  $  2630.40$ &  $\mathbf{   1225.40 }$ &  $           1199.60 $ \\
             Bowling & $    23.10$ &  $   160.70$ &  $              68.10 $ &  $\mathbf{    102.10 }$ \\
              Boxing & $     0.10$ &  $    12.10$ &  $              91.60 $ &  $\mathbf{     99.30 }$ \\
            Breakout & $     1.70$ &  $    30.50$ &  $\mathbf{    418.50 }$ &  $            344.10 $ \\
           Centipede & $  2090.90$ &  $ 12017.00$ &  $            5409.40 $ &  $\mathbf{  49065.80 }$ \\
     Chopper Command & $   811.00$ &  $  7387.80$ &  $\mathbf{   5809.00 }$ &  $            775.00 $ \\
       Crazy Climber & $ 10780.50$ &  $ 35829.40$ &  $          117282.00 $ &  $\mathbf{ 119679.00 }$ \\
            Defender & $  2874.50$ &  $ 18688.90$ &  $\mathbf{  35338.50 }$ &  $          11099.00 $ \\
        Demon Attack & $   152.10$ &  $  1971.00$ &  $           58044.20 $ &  $\mathbf{  63644.90 }$ \\
         Double Dunk & $   -18.60$ &  $   -16.40$ &  $\mathbf{     -5.50 }$ &  $            -11.50 $ \\
              Enduro & $     0.00$ &  $   860.50$ &  $            1211.80 $ &  $\mathbf{   2002.10 }$ \\
       Fishing Derby & $   -91.70$ &  $   -38.70$ &  $              15.50 $ &  $\mathbf{     45.10 }$ \\
             Freeway & $     0.00$ &  $    29.60$ &  $              33.30 $ &  $\mathbf{     33.40 }$ \\
           Frostbite & $    65.20$ &  $  4334.70$ &  $            1683.30 $ &  $\mathbf{   3469.60 }$ \\
              Gopher & $   257.60$ &  $  2412.50$ &  $           14840.80 $ &  $\mathbf{  56218.20 }$ \\
            Gravitar & $   173.00$ &  $  3351.40$ &  $             412.00 $ &  $\mathbf{    483.50 }$ \\
            H.E.R.O. & $  1027.00$ &  $ 30826.40$ &  $\mathbf{  20130.20 }$ &  $          14225.20 $ \\
          Ice Hockey & $   -11.20$ &  $     0.90$ &  $\mathbf{     -2.70 }$ &  $             -4.10 $ \\
          James Bond & $    29.00$ &  $   302.80$ &  $\mathbf{   1358.00 }$ &  $            507.50 $ \\
            Kangaroo & $    52.00$ &  $  3035.00$ &  $           12992.00 $ &  $\mathbf{  13150.00 }$ \\
               Krull & $  1598.00$ &  $  2665.50$ &  $            7920.50 $ &  $\mathbf{   9745.10 }$ \\
      Kung-Fu Master & $   258.50$ &  $ 22736.30$ &  $           29710.00 $ &  $\mathbf{  34393.00 }$ \\
 Montezuma's Revenge & $     0.00$ &  $  4753.30$ &  $\mathbf{      0.00 }$ &  $\mathbf{      0.00 }$ \\
          Ms. Pacman & $   307.30$ &  $  6951.60$ &  $            2711.40 $ &  $\mathbf{   4963.80 }$ \\
      Name This Game & $  2292.30$ &  $  8049.00$ &  $           10616.00 $ &  $\mathbf{  15851.20 }$ \\
             Phoenix & $   761.40$ &  $  7242.60$ &  $\mathbf{  12252.50 }$ &  $           6202.50 $ \\
             Pitfall & $  -229.40$ &  $  6463.70$ &  $             -29.90 $ &  $\mathbf{     -2.60 }$ \\
                Pong & $   -20.70$ &  $    14.60$ &  $\mathbf{     20.90 }$ &  $             20.60 $ \\
         Private Eye & $    24.90$ &  $ 69571.30$ &  $             129.70 $ &  $\mathbf{    286.70 }$ \\
              Q*Bert & $   163.90$ &  $ 13455.00$ &  $\mathbf{  15088.50 }$ &  $           5236.80 $ \\
          River Raid & $  1338.50$ &  $ 17118.00$ &  $\mathbf{  14884.50 }$ &  $          12530.80 $ \\
         Road Runner & $    11.50$ &  $  7845.00$ &  $           44127.00 $ &  $\mathbf{  47770.00 }$ \\
            Robotank & $     2.20$ &  $    11.90$ &  $\mathbf{     65.10 }$ &  $             64.30 $ \\
            Seaquest & $    68.40$ &  $ 42054.70$ &  $\mathbf{  16452.70 }$ &  $          10932.30 $ \\
              Skiing & $-17098.10$ &  $ -4336.90$ &  $\mathbf{  -9021.80 }$ &  $         -13585.10 $ \\
             Solaris & $  1236.30$ &  $ 12326.70$ &  $            3067.80 $ &  $\mathbf{   4544.80 }$ \\
      Space Invaders & $   148.00$ &  $  1668.70$ &  $            2525.50 $ &  $\mathbf{   2589.70 }$ \\
         Star Gunner & $   664.00$ &  $ 10250.00$ &  $\mathbf{  60142.00 }$ &  $            589.00 $ \\
            Surround & $   -10.00$ &  $     6.50$ &  $              -2.90 $ &  $\mathbf{     -2.50 }$ \\
              Tennis & $   -23.80$ &  $    -8.30$ &  $             -22.80 $ &  $\mathbf{     12.10 }$ \\
          Time Pilot & $  3568.00$ &  $  5229.20$ &  $\mathbf{   8339.00 }$ &  $           4870.00 $ \\
           Tutankham & $    11.40$ &  $   167.60$ &  $\mathbf{    218.40 }$ &  $            183.90 $ \\
         Up and Down & $   533.40$ &  $ 11693.20$ &  $\mathbf{  22972.20 }$ &  $          22474.40 $ \\
             Venture & $     0.00$ &  $  1187.50$ &  $              98.00 $ &  $\mathbf{   1172.00 }$ \\
       Video Pinball & $ 16256.90$ &  $ 17667.90$ &  $\mathbf{ 309941.90 }$ &  $          56287.00 $ \\
       Wizard of Wor & $   563.50$ &  $  4756.50$ &  $\mathbf{   7492.00 }$ &  $            483.00 $ \\
        Yars Revenge & $  3092.90$ &  $ 54576.90$ &  $           11712.60 $ &  $\mathbf{  21409.50 }$ \\
              Zaxxon & $    32.50$ &  $  9173.30$ &  $           10163.00 $ &  $\mathbf{  14402.00 }$ 
\end{tabular}
\caption{\label{raw_scores} Raw scores for a random agent, a human tested, Double DQN as described by \citet{vanHasselt:2016}, and Double DQN with Pop-Art and no reward clipping on 30 minutes of simulated play (108,000 frames). The random, human, and Double DQN scores are all taken from \citet{Wang:2016}.}
\end{center}
\end{table*}

\end{document}